%% file: main.tex
\documentclass{article}
\usepackage{ijcai21}

\usepackage{times}

\usepackage{soul}
\usepackage{url}
\usepackage[hidelinks]{hyperref}
\usepackage[utf8]{inputenc}
\usepackage[small]{caption}
\usepackage{graphicx}
\usepackage{amsmath}
\usepackage{booktabs}
\usepackage{nicefrac}
\usepackage{amsthm}
\usepackage{amssymb}
\usepackage{xcolor}
\usepackage[ruled,linesnumbered]{algorithm2e}

\usepackage{enumitem}
\usepackage{subcaption}

\newtheorem{theorem}{Theorem}
\newtheorem{proposition}{Proposition}

\newtheorem{lemma}{Lemma}

\newtheorem{example}{Example}
\theoremstyle{remark}
\newtheorem*{remark}{Remark}

\title{Online Selection of Diverse Committees\footnote{A short version of this paper appears in the Proceedings of IJCAI-2021.}}

\author{
Virginie Do$^{1,2}$\footnote{Contact Author: virginie.do@dauphine.eu}\and
Jamal Atif$^1$\and
Jérôme Lang$^1$\And
Nicolas Usunier$^2$\\
\affiliations
$^1$LAMSADE, Université PSL, Université Paris-Dauphine, CNRS, France\\
$^2$Facebook AI Research}

\input{aliases.tex}

\begin{document}

\maketitle
\begin{abstract}

Citizens' assemblies need to represent subpopulations according to their proportions in the general population. These large committees are often constructed in an online fashion by contacting people, asking for the demographic features of the volunteers, and deciding to include them or not. This raises a trade-off between the number of people contacted (and the incurring cost) and the representativeness of the committee. We study three methods, theoretically and experimentally: a greedy algorithm that includes volunteers as long as proportionality is not violated; a non-adaptive method that includes a volunteer with a probability depending only on their features, assuming that the joint feature distribution in the volunteer pool is known; and a reinforcement learning based approach when this distribution is not known a priori but learnt online.  
\end{abstract}

\section{Introduction}

Forming a representative committee consists in selecting a set of individuals, who agree to serve, in such a way that every part of the population, defined by specific features, is represented proportionally to its size. As a paradigmatic example, the Climate Assembly 
in the UK and the Citizens' Convention for Climate in France brought together 108 and 150 participants respectively, representing sociodemographic categories such as gender, age, education level, professional activity, residency, and location, in proportion to their importance in the wider society.
Beyond citizens' deliberative assemblies, proportional representation often has to be respected when forming an evaluation committee, selecting a diverse pool of students or employees, and so on.

Two key criteria for evaluating the committee formation process are the representativeness of the final selection 
and the
number of persons contacted (each of these incurring a cost).
The trade-off is that the higher the number of people contacted, the more proportional the resulting committee.

A first possibility is to use an offline strategy (as for the UK assembly):  invitations are sent to a large number of people (30,000), and the final group is selected among the pool of volunteers. An alternative setting which is common in hiring is to consider an online process: the decision-maker is given a stream of candidates and has to decide at each timestep whether or not to admit the candidate to the final committee. This work focuses on the latter setting.


A further difficulty is that the distribution of \emph{volunteers} is not necessarily known in advance. For example, although the target is to represent distinct age groups proportionally to their distribution in the wider population, it may be the case that older people are predominant among volunteers.

Multi-attribute proportional representation in committee selection in an off-line setting usually assumes full access to a finite (typically large) database of candidates. This assumption is impractical in a variety of real-world settings: first, the database does not exist beforehand and constructing it would require contacting many more people than necessary; second, in some domains, the decision to hire someone should be made immediately so that people don't change their mind 
in the meantime (which is typical 
in professional contexts).

An online strategy must achieve a good trade-off between sample complexity, i.e. the number of timesteps needed to construct a full committee, and the quality of the final committee, as measured by its distance to the target distribution. 


We focus on the online setting. We introduce a new model and offer three different strategies, which rely on different assumptions on the input (and the process). The {\em greedy} strategy selects volunteers
as long as their inclusion does not jeopardize the size and representation constraints; it does not assume any prior distribution on the volunteer pool. The {\em nonadaptive} strategy, based on constrained Markov decision processes, repeatedly chooses a random person, 
and decides whether to include or not a volonteer with a probability that depends only on their features; it assumes the joint distribution in the volunteer pool is known; it can be parallelised. Finally, the {\em reinforcement learning} strategy assumes this distribution is not known a priori but can be learnt online.

Which of these 
strategies are interesting depends on domain specificities. For each, we study bounds for expected quality and sample complexity, and perform 
experiments
using real data from the UK Citizens' Assembly on Brexit. 

The outline of the paper is as follows. 
We discuss related work in Section \ref{sec:related}, define the problem in Section \ref{sec:model}, define and study 
our three strategies 
in Sections  \ref{sec:greedy}, \ref{sec:cmdp} and \ref{sec:learn}, 
analyse our experiments in Section \ref{sec:expes} and conclude in Section \ref{sec:conclu}.

\input{sections/related_work}

\input{sections/problem_def}
\input{sections/mdp}

\input{sections/learning.tex}

\input{sections/experiments}

\section{Conclusion}\label{sec:conclu}
We formalised the problem of selecting a diverse committee with multi-attribute proportional representation in an online setting. We addressed the case of known candidate distributions with constrained MDPs, and leveraged exploration-exploitation techniques to address unknown distributions. 


\clearpage
\section*{Acknowledgements}

This work was funded in part by the French government under management of Agence Nationale de la Recherche as part of the  ``Investissements  d’avenir'' program  ANR-19-P3IA-0001 (PRAIRIE 3IA Institute). We thank Matteo Pirotta for his helpful suggestions and feedback.

\bibliographystyle{named}
\bibliography{ijcai21}

\clearpage

\appendix

\input{sections/greedy}
\input{sections/proofs}

\input{sections/bernstein}
\input{sections/additional_exps}
\input{sections/example-cmdp}

\end{document}

%% file: aliases.tex
\usepackage{bm}
\usepackage{xspace}
\usepackage{dsfont}

\newcommand{\expe}[2]{\mathbb{E}_{#1}\left[{#2}\right]}

\newcommand{\pr}[2]{\mathbb{P}_{#1}\left[{#2}\right]}

\newcommand{\ratio}{\lambda}
\newcommand{\target}{\rho}
\newcommand{\tol}{\epsilon}
\newcommand{\cstr}{\xi}
\newcommand{\reals}{\mathbb{R}}
\newcommand{\actS}{\mathcal{A}}

\newcommand{\optpi}{\pi^*}
\newcommand{\size}{D}

\newcommand{\nepi}{L} 
\newcommand{\iepi}{l} 
\newcommand{\tepi}{\tau}
\newcommand{\durepi}{E_\iepi}

\newcommand{\regret}{R}
\newcommand{\violate}{R^c}

\newcommand{\card}[1]{|#1|}
\newcommand{\intint}[1]{[#1]}
\newcommand{\indic}[1]{\mathds{1}_{\!\{#1\}}}
\newcommand{\candidateS}{\mathcal{X}}
\newcommand{\loss}[1]{\|\ratio(#1) - \target\|_\infty}
\newcommand{\ucb}{\overline{p}}
\newcommand{\lcb}{\underline{p}}
\newcommand{\est}{\hat{p}}
\newcommand{\varp}{\tilde{p}}
\newcommand{\rad}{\beta}

\newcommand{\ncand}{\card{\candidateS}}
\newcommand{\stateS}{\mathcal{X}}

\newcommand{\ceil}[1]{\lceil #1 \rceil}
\newcommand{\floor}[1]{\lfloor #1 \rfloor}
\newcommand{\abs}[1]{\left\vert #1 \right\vert}

\newcommand{\ALG}{ALG\xspace}
\newcommand{\greedyalg}{\texttt{Greedy}\xspace}
\newcommand{\statioalg}{\texttt{CMDP}\xspace}
\newcommand{\optalg}{\texttt{RL-CMDP}\xspace}
\newcommand{\optalgb}{\texttt{RL-CMDP-B}\xspace}

\newcommand{\sumepi}{\sum_{t\in \durepi}}
\newcommand{\supp}{\mathcal{S}}

\newcommand{\prepi}{\tilde{p}_\iepi}
\newcommand{\piepi}{\pi_\iepi}
\newcommand{\piepit}{\pi_{\iepi_t}}
\newcommand{\filtr}{\mathcal{F}}
\newcommand{\radepi}{\beta_\iepi}

%% file: sections/related_work.tex
\section{Related work} \label{sec:related}

\paragraph{Diversity and representation in committee (s)election}

The problem of selecting a diverse set of candidates 
from a 
candidate database, where each candidate is described by a vector of attribute values, has been considered in several places. In \cite{lang2018multi}, the goal is to find a committee of a fixed size whose distribution of attribute values is as close as possible to a given target distribution.
In \cite{celis2018multiwinner,bredereck2018multiwinner}, each candidate has a score, obtained from a set of votes, and some constraints on the proportion of selected candidates with a given attribute value are specified; the goal is to find a fixed-size committee of maximal score satisfying the constraints. In the same vein, \cite{Aziz19} considers soft constraints, and \cite{BeiLPW20} do not require the  size of the committee to be fixed.\footnote{
Note that {\em diversity} and {\em proportional representation} are often used with a  
different meaning in multiwinner elections, namely, in the sense that each voter should feel represented in an elected committee, regardless of attributes. A good entry to this literature is the survey \cite{faliszewski2017multiwinner}.
}

Our online setting shifts the difficulty of the multi-attribute representation problem from computational complexity analyses, to the need for probabilistic guarantees on the tradeoffs between sample complexity and achieved proportionality.

\paragraph{Representative and fair sortition}

Finding a representative committee (typically, a panel of citizens) with respect to a set of attributes, using {\em sortition}, is the topic of at least two recent papers. \cite{benade2019no} show that stratification (random selection from small subgroups defined by attribute values, rather than from the larger group) only helps marginally. 
\cite{FlaniganGGP20} go further and consider this three-stage selection process: (1) letters are sent to a large number of random individuals (the {\em recipients}); (2) these recipients answer whether they agree to participate, and if so, give their features; those individuals constitute the {\em pool}; (3) a sampling algorithm is used to select the final {\em panel} from the pool. As the probability of willingness to participate is different across demographic groups, each person is selected with a probability that depends on their features, so as to correct this self-selection bias. This guarantees that the whole process be fair to all individuals of the population, with respect of going from the initial population to the panel.\footnote{Fairness guarantees are pushed further in following (yet unpublished) work by the authors: see \url{https://youtu.be/x_1Ce1kT7vc}.}

The main differences between this work and ours are: (1) (once again)  
our process is online; (2) we do not consider individual fairness, only group representativeness;
(3) we care about minimizing the number of people contacted.
Moreover, unlike off-line processes, our process can be applied in contexts where hiring a person just  interviewed cannot be delayed; this may not be crucial for citizens' assemblies (although someone who volunteers at first contact may change their mind if the delay until the final selection is long), but this is definitely so when hiring a diverse 
team of employees. 




\paragraph{Online selection problems} 

Generalized secretary problems \cite{babaioff2008online} are 
optimal stopping problems where the goal is to hire the best possible subset of persons, assuming that persons arrive one at a time, their value is observed at that time,
and the decision to hire or not them must be taken immediately. The problem has been generalized to finding a set of items maximizing a submodular value function \cite{bateni2013submodular,badanidiyuru2014streaming}
While the latter models do not deal with diversity constraints, 
\cite{stoyanovich2018online} aims at selecting a group of people arriving in a streaming fashion from a finite pool, with the goal of optimizing their overall quality subject to diversity constraints. The common point with our approach is the online nature of the selection process.
The main differences are that they consider only one attribute, the size of the pool is known, and yet more importantly, what is optimized is the intrinsic quality values of the candidates and not the number of persons interviewed.
Closer to our setting is \cite{panigrahi2012online} who consider diversity along multiple features in online selection of search results, regardless of item quality. 
They only seek to maximise diversity, and do not consider trade-offs with the number of items observed. 

The diverse hiring setting of \cite{schumann2019diverse} is very different. At each time step, the decision-maker chooses which candidate to interview and only decides on which subset to hire after multiple rounds, whereas in our setting, candidates arrive one by one and decisions are made immediately.

%% file: sections/problem_def.tex
\section{Formal setting}\label{sec:model}

\subsection{Problem definition}
Let $\candidateS = \candidateS_1 \times ... \times \candidateS_d$ be the product space of $d$ finite domains, each of size $\size_i = \card{\candidateS_i}$, and where we identify $\candidateS_i$ with $\intint{\size_i}= \{1, ..., \size_i\}$. Each candidate is represented by a \emph{characteristic vector} $x \in \candidateS$ with $d$ \emph{features}. Let $x^i \in \candidateS_i$ denote the value of the $i$-th feature. For each $i \in \intint{d}$, we consider a \emph{target vector} $\target^i \in (0,1)^{\size_i}$ with $\sum_{j=1}^{\size_i}\target_j^i = 1.$

The candidate database is infinite and the horizon as well. At each timestep $t \geq 1$, the agent observes a candidate $x_t$ drawn i.i.d.\ from a
stationary distribution $p$ over $\candidateS$, i.e. $x_t \sim p$. The decision-maker must immediately decide between two actions: \emph{accept} or \emph{reject} the candidate, which we respectively denote as $a_t=1$ and $a_t=0$.

The goal 
is to select a \emph{committee} $C$ of $K$ candidates that matches the target vectors as closely as possible, while minimizing the number of candidates screened.

For some set $C$, let $\ratio(C) \in \prod_{i=1}^d[0,1]^{\size_i}$ be the \emph{representation profile} of $C$, where $\ratio_j^i(C) = \frac{\card{\{x \in C: x^i = j\}}}{\card{C}}$. We define the \emph{representation loss} as $\|\ratio(C) - \target\|_\infty = \max_{i \in \intint{d}, j \in \intint{\size_i}} \vert\ratio_j^i(C) - \target_j^i\vert$. We evaluate how much $C$ matches the target $\target$ by the $\ell_\infty$ metric, because it is harsher than $\ell_1, \ell_2$ on committees that are unacceptable in our applications (e.g. committees with no women that achieve perfect representation on all other categories than gender).

Let $C_t = \{x_{t'}: t' \leq t, a_{t'} = 1\}$ denote the set of all accepted candidates at the end of step $t$. The agent stops at $\tau$, where $\tau$ is the first time when $K$ candidates have been accepted, i.e. the total number of candidates screened. The agent following a (possibly randomized) algorithm $\ALG$ must minimize the \emph{sample complexity} $\mathbb{E}^{p, \ALG}[\tau]$.
 
Importantly, we consider two settings: whether the candidate distribution $p$ is \emph{known} or \emph{unknown}.

\begin{remark}
In this model, we simply ignore non-volunteers, since the agent only needs to make decisions for volunteers, which from now on we call \emph{candidates}. The joint distribution of characteristic vectors in the population of candidates is $p$.
\end{remark}

\subsection{Greedy strategy}\label{sec:greedy}

We describe a first simple strategy. In \greedyalg, the agent greedily accepts any candidate as long as the number of people in the committee with $x^i = j$ does not exceed the quota $\ceil{\target_j^i K} + \frac{\tol K}{(\size_i - 1)}$  for any $i,j$, where $\tol > 0$ is some tolerance parameter for the representation quality. 

\begin{proposition}\label{prop:greedy-loss}
The representation loss incurred by \greedyalg is bounded as follows:
\begin{align*}
    \loss{C_\tau} \mathop{\leq}_{a.s.} \big(\frac{\max_{i\in[d]}\size_i - 1}{K} + \tol).
\end{align*}
\end{proposition}
The proof and pseudocode are included in App. \ref{app:greedy}.

This method is simple to interpret and implement, and can even be used when the candidate distribution $p$ is unknown. However, in the following example, we see that \greedyalg may be inefficient because it requires interacting with an arbitrarily large number of candidates to recruit a full committee. 


\begin{table}[]
\centering
\begin{tabular}{|l|ll|}
\hline
gender \textbackslash{} age & S & J \\
\hline
M & $\nicefrac12-\epsilon'$ & $\nicefrac14$ \\
F & $\nicefrac14$ & $\epsilon'$\\
\hline
\end{tabular}
\caption{Example candidate distribution $p$ with 2 binary features.}
\label{tab:example-dist}
\end{table}

\begin{example}\label{ex:small-committee}
Let $\epsilon' >0, \ll 1$. There are 2 binary features, \emph{gender} and \emph{age}, with domains $\candidateS_{\text{gender}}=\{M,F\}$ 
and $\candidateS_{\text{age}}=\{S,J\}$.
The candidates are distributed as $p$ given in Table \ref{tab:example-dist}. We want a committee of size $K = 4$ (e.g., a thesis committee) and the target is $\target^{\text{gender}}=(\nicefrac{1}{2}, \nicefrac{1}{2})$ and $\target^{\text{age}} = (\nicefrac{3}{4}, \nicefrac{1}{4})$.

Let $A$ be the event that in the first $3$ timesteps, the agent observes candidates with characteristic vectors $\{FS, MS, MS\}$ in any order. Then \greedyalg accepts all of them, i.e. $A = \left\{C_3 = \{FS, MS, MS\}\right\}$. We have: $\pr{}{A} = \nicefrac{1}{4} (\nicefrac{1}{2} - \epsilon')^2\times 3!  = \nicefrac{3}{2} (\nicefrac{1}{2} - \epsilon')^2 \geq \nicefrac{3}{2} \big(\nicefrac{1}{3}\big)^2 = \nicefrac16$.

Under event $A$, \greedyalg can only stop upon finding $FJ$ in order to satisfy the representation constraints. Therefore, $\tau | A$ follows a geometric distribution with success probability $\epsilon'$, hence its expectation is $\nicefrac{1}{\epsilon'}$, and
$\mathbb{E}^{p,\greedyalg}[\tau] \geq \expe{}{\tau|A}\times\pr{}{A} = \nicefrac{1}{6\epsilon'}.$
Therefore, the sample complexity of \greedyalg in this example is arbitrarily large. 
\end{example}

This example shows the limits of directly applying a naive strategy to our online selection problem, where the difficulty arises from considering multiple features simultaneously, even when there are only $2$ binary features. We further discuss the strengths and weaknesses of \greedyalg, and its sensitivity to the tolerance $\tol$ in our experiments in Section \ref{sec:expes}.

The greedy strategy is adaptive, in the sense that decisions are made based on the current candidate and candidates accepted in the past. In the following section, we present, with theoretical guarantees, an efficient yet non-adaptive algorithm based on constrained MDPs for the setting in which the candidate distribution is known. We then adapt this approach to the case when this distribution is unknown, using techniques for efficient exploration / exploitation in constrained MDPs relying on the principle of optimism in the face of uncertainty.

%% file: sections/mdp.tex
\section{$p$ is known: constrained MDP strategy} \label{sec:cmdp}

In this section, we assume the 
distribution $p$ is known, and we place ourselves in the limit where we would select a committee of infinite size, and aim to maximize the rate at which candidates are selected, under the constraint that the proportion of accepted candidates per feature value is controlled by $\target$. One advantage of this approximation is that the optimal policy is stationary, thus simple to represent. Moreover, 
as stationary policies can be very well 
parallelized, in the case where 
multiple candidates can be interviewed simultaneously. To apply this approach to the finite-size committee selection problem, one needs to interrupt the agent when $K$ candidates have been selected. We showcase a high probability bound of $O(\sqrt{1/K})$ on the representation loss, which guarantees that for large enough values of $K$, the resulting committee is representative. 

From now on, we assume that any feature vector can be observed, i.e., $p(x) > 0$ for all $x$, so that proportional representation constraints can be satisfied.

\subsection{Our model} 
Fundamentally, our problem could be seen as a contextual bandit with stochastic contexts $x_t \sim p$ and two actions $a_t = 0$ or $1$. However, the type of constraints incurred by proportional representation are well studied in constrained MDPs (CMDPs) \cite{altman1999constrained}, whereas the contextual bandits literature focused on other constraints (e.g., knapsack constraints \cite{agrawal2016linear}). We show how we can efficiently leverage the CMDP framework for our online committee selection problem.

Formally, we introduce an MDP $M=(\candidateS, \actS, P, r)$, where the set of states is the $d$-dimensional candidate space $\candidateS$, the set of actions is $\actS = \{0, 1\}$, and the (deterministic) reward is $r(x,a) = \indic{a=1}$. 
The transition kernel $P$, which defines the probability to be in state $x'$ given that the previous state was $x$ and the agent took action $a$, is very simple in our case: we simply have $P(x'|x, a) = p(x')$ since candidates are drawn i.i.d regardless of the previous actions and candidates.

We consider the \emph{average reward} setting in which the performance of a policy $\pi:\stateS \times \actS \rightarrow [0, 1]$ is measured by its \emph{gain} $g^{p,\pi}$, defined as:
\begin{align*}
g^{p,\pi}(x) = \lim_{T \rightarrow \infty} \frac{1}{T} \mathbb{E}^{p,\pi}\left[\sum_{t=1}^T r(x_t, a_t) \bigg\vert x_1 = x\right].
\end{align*}
We simply write $g^{p,\pi} := g^{\pi}$ when the underlying transition is $p$ without ambiguity.

We include proportional representation constraints following the framework of CMDPs, where the set of allowed policies is restricted by a set of additional constraints specified by reward functions. In our case, for $i \in \intint{d}, j \in \intint{\size_i}$, we introduce $r_j^i(x, a) = \indic{x^i=j, a=1}$, and let $\cstr_j^i = r_j^i - \target_j^i r$ be the reward function for the constraint indexed by $i,j$. Similarly to the gain, we define ${h_j^i}^\pi = \lim_{T \rightarrow \infty} \frac{1}{T} \mathbb{E}^\pi\left[\sum_{t=1}^T \cstr_j^i(x_t, a_t)\right]$.
The CMDP is defined by:
\begin{align}\label{eq:CMDP-opt1}
    \max_{\pi} \{g^\pi \,| \,\forall i \in \intint{d}, \forall j \in \intint{\size_i}, {h_j^i}^\pi = 0 \}.
\end{align}

Given the simplicity of the transition kernel, and since the MDP is ergodic by the assumption $p > 0$, the gain is constant, i.e. $\forall x \in \candidateS, g^\pi(x) = g^\pi$, and problem \eqref{eq:CMDP-opt1} is well defined. From now on, we only write $g^\pi$ and ${\cstr_j^i}^\pi$. Moreover, the optimal policy for the CMDP \eqref{eq:CMDP-opt1} is denoted $\optpi$ and is \emph{stationary} \cite{altman1999constrained}.

\begin{lemma}\label{lem:constraint2target}
$g^\pi$ is the \emph{selection rate} under policy $\pi$: 
\begin{align*}
    g^\pi = \sum_{x}p(x)\pi(x,1) = \mathbb{P}^{p,\pi}[a=1]
\end{align*}

Moreover, if $\pi$ is feasible for CMDP \eqref{eq:CMDP-opt1}, then:  
\begin{align*}\forall i \in [d], \forall j \in \intint{\size_i}, \mathbb{P}^{p,\pi}[x^i=j | a=1] =\target_j^i.
\end{align*}
\end{lemma}

Lemma \ref{lem:constraint2target} implies that (a) $\pi^*$ maximises the selection rate of candidates, and (b) the constraints of \eqref{eq:CMDP-opt1} force candidates $x$ with $x^i = j$ to be accepted in proportions given by $\target_j^i$.

The CMDP can be expressed as the 
linear program: 
\begin{equation}\label{eq:CMDP-opt2}
\begin{aligned}
\max_{\pi \in \reals_+^{\candidateS \times \actS}} \quad & \sum_{x,a}\pi(x,a)p(x) r(x,a)  \\
\text{u.c.} \quad & \forall x\in \candidateS, \sum_{a} \pi(x,a)  = 1 \\
& \forall i,j , \sum_{x,a}\pi(x,a)p(x) \cstr_j^i(x,a) = 0.
\end{aligned}
\end{equation}

Notice that problem \eqref{eq:CMDP-opt2} is feasible by the assumption that $\forall x\in \candidateS, p(x) > 0$. Next we study how well the proportional selection along features is respected when we shift from infinite to finite-sized committee selection.

\subsection{Theoretical guarantees}

We analyze the \texttt{CMDP-based} strategy where at each timestep, the agent observes candidates $x_t \sim p$, decides to accept $x_t$ by playing $a_t \sim \optpi(.|x_t)$ and stops when $K$ candidates have been accepted. We later refer to it as \statioalg for brevity.

First, we formally relate the gain $g^\pi$ that we optimize for in \eqref{eq:CMDP-opt1} to the quantity of interest $\mathbb{E}^{p,\pi}[\tau]$.
\begin{lemma}\label{lem:gain2time}
For any stationary policy $\pi$, $\mathbb{E}^{p,\pi}[\tau] = \frac{K}{g^\pi}$.
\end{lemma} 
Lemma \ref{lem:gain2time} is a direct consequence of the fact that $\tau + K$ follows a negative binomial distribution with parameters $K$ and $1-g^\pi$, which are respectively the number of successes and the probability of failure, i.e. of rejecting a candidate under $\pi$. Note that this is only true because in our case the transition structure of the MDP ensures constant gain. A quick sanity check shows that if the agent systematically accepts all candidates, i.e. $g^\pi = 1$, then $\mathbb{E}^{p,\pi}[\tau] = K$, and that maximizing $g^\pi$ is equivalent to minimizing $\mathbb{E}^{p,\pi}[\tau]$. 

We exhibit a bound on the representation loss of \statioalg which follows the optimal stationary policy $\optpi$ of CMDP \eqref{eq:CMDP-opt1}.  
Let $\tilde{d} = \sum_{i=1}^d (D_i - 1).$ 
($\tilde{d} = d$ when all features are binary.)

\begin{proposition}\label{prop:hoeffding-loss}
Let $\pi^*$ be an optimal stationary policy for CMDP \eqref{eq:CMDP-opt1}. Let $\delta > 0$. Then, 
\vspace{-1mm}
\begin{align*}
    \mathbb{P}^{p,\pi^*}\left[\loss{C_\tau} \leq \sqrt{\frac{\log(\frac{2\tilde{d}}{\delta})}{2K}}\right] \geq 1 - \delta.
\end{align*}
\vspace{-1mm}
\end{proposition}

All proofs of this section are available in Appendix \ref{sec:proof-cmdp}.

The upper bound on the representation loss of \statioalg decreases with the committee size in $\sqrt{1/K}$. This shows that the stationary policy $\optpi$ works well for larger committees, although it acts independently from previously accepted candidates. The intuition is that for larger committees, adding a candidate has less impact on the current representation vector. 
\begin{example}\label{ex:cmdp}
We take the same attributes and same distribution as in Table \ref{tab:example-dist}, with $\epsilon' = \nicefrac16$. Here, the target vectors are $\target^{\text{gender}} = (\nicefrac12,\nicefrac12)$ and $\target^{\text{age}} = (\nicefrac12,\nicefrac12)$: an ideal committee contains as many women as men, as many senior as junior. 

With the optimal policy for LP \eqref{eq:CMDP-opt2}, each time the current volunteer is a senior male, we select him with probability \nicefrac12; all other volunteers are selected with probability 1. The expected final composition of the pool is 30\% of junior male, 30\% of senior female, 20\% of junior female and 20\% of senior male. As the policy selects in average \nicefrac56 of the volunteers, the expected time until we select $K$ candidates is $\mathbb{E}^{p,\pi^*}[\tau] = (\nicefrac65) K$. More details can be found in App. \ref{appendix:cmdp}.
\end{example}

%% file: sections/learning.tex
\section{$p$ is unknown: optimistic CMDP strategy} \label{sec:learn} 

\begin{algorithm}[t]
\caption{\optalg algorithm. \label{alg:opt}}
 \SetKwInOut{Input}{input}\SetKwInOut{Output}{output}
 \Input{confidence $\delta$, committee size $K$, targets $\target$}
 \Output{committee $C_\tau$}
 $t\gets 0$, $C_0 \gets \emptyset$\;
 \While{$\card{C_t} < K$ }{
 \For{episode $\iepi = 1, 2,...$}{
    $\tepi_\iepi = t + 1$\;
    $\pi_\iepi \gets$ sol. of \eqref{eq:OptCMDP-1} via the extended LP \eqref{eq:OptCMDP-2}\;
    \While{$n_t(x_t) < 2 n_{\tepi_\iepi - 1}(x_t)$}{
        $t \gets t + 1$, Execute $\pi_\iepi$\;
    }
 }
 }
\Return{$C_t$}
\end{algorithm}

We now tackle the committee selection problem when the candidate distribution $p$ is unknown and must be learned online. Let $g^* = g^{\pi^*}$ be the value of \eqref{eq:CMDP-opt1}, which is the optimal gain of the CMDP when the distribution $p$ is known. We evaluate a learning algorithm by:
\begin{enumerate}
    \item the performance regret: $\regret(T) = \sum_{t=1}^T (g^* - r(x_t, a_t))$,
    \item the cost of constraint violations:
    \begin{sloppypar}
    $\violate(T) =\max_{i,j} \big\vert\sum_{t=1}^T \xi_j^i(x_t, a_t)\big\vert$.
    \end{sloppypar}
\end{enumerate}


We propose an algorithm that we call \optalg (Reinforcement Learning in CMDP, Alg.~\ref{alg:opt}). It is an adaptation of the \emph{optimistic} algorithm UCRL2 \cite{jaksch2010near}, and it also builds on the algorithm OptCMDP proposed by \cite{efroni2020exploration}  for finite-horizon CMDPs. Learning in average-reward CMDPs involves different challenges, because there is no guarantee that the policy at each episode has constant gain. It does not matter in our case, since as we noted in Sec. \ref{sec:cmdp}, the simple structure of the transition kernel ensures constant gain, and does not require to use the Bellman equation. The few works on learning in average-reward CMDPs make unsuitable assumptions for our setting \cite{zheng2020constrained,singh2020learning}. 

\optalg proceeds in episodes, which end each time the number of observations for some candidate $x$ doubles. During each episode $\iepi$, observed candidates $x_t$ are accepted on the basis of a single stationary policy $\pi_\iepi$.

Let $\tepi_\iepi$ denote the start time of episode $\iepi$ and $\durepi = [\tepi_\iepi, \tepi_{\iepi+1}]$. Let $n_t(x) = \sum_{t'=1}^t \indic{x_{t'}=x}$ and $N(t) = \card{C_{t-1}} = \sum_{t'=1}^{t-1}\indic{a_{t'} = 1}$. Let $N_j^i(t) = \sum_{t'=1}^{t-1}\indic{x^i_{t'} =j, a_{t'} = 1}$ be the number of accepted candidates $x$ such that $x^i = j$ before $t$.


At each episode $l$, the algorithm estimates the true candidate distribution by the empirical distribution $\est_\iepi(x) = \frac{n_{\tepi_\iepi-1}(x)}{\tepi_\iepi-1}$ and maintains confidence sets $B_l$ on $p$. As in UCRL2, these are built using the inequality on the $\ell_1$-deviation of
$p$ and $\est_\iepi$ from \cite{weissman2003inequalities}:

\begin{lemma}\label{lem:l1-dev}
With probability $\geq 1-\frac{\delta}{3}$,
\begin{equation}\label{eq:L1confidence}
\begin{aligned}
    \|\hat{p}_\iepi - p\|_{1} \leq \sqrt{\frac{ 2\ncand\log\big(6\ncand \tepi_\iepi (\tepi_\iepi-1) / \delta\big)}{\tepi_\iepi -1}} := \rad_\iepi
\end{aligned}
\end{equation}
\end{lemma}

Let $B_l = \{\tilde{p} \in \Delta(\candidateS): \|\hat{p}_\iepi - \tilde{p}\|_{1} \leq  \rad_\iepi\}$ be the confidence set for $p$ at episode $\iepi$. The associated set of compatible CMDPs is then $\{\tilde{M} = (\candidateS, \actS, \varp, r, \xi):  \varp \in B_\iepi\}$. At the beginning of each episode, \optalg finds the optimum of: \begin{align}\label{eq:OptCMDP-1}
    \max_{\pi \in \Pi, \varp \in B_\iepi} \{g^{\varp, \pi} \,| \,\forall i,j,\, {h_j^i}^{\varp, \pi} = 0 \}.
\end{align}


\paragraph{Extended LP} In order to optimize this problem, we re-write \eqref{eq:OptCMDP-1} as an extended LP. Following \cite{rosenberg2019online} and the CMDP literature, we introduce the state-action occupation measure $\mu(x,a) = \pi(x,a) p(x)$ and variables $\beta(x)$ to linearize the $\ell_1$ constraint induced by the confidence set:

\begin{equation}\label{eq:OptCMDP-2}
\begin{aligned}
\max_{\substack{\mu \in \reals^{\candidateS \times \actS } \\ \rad \in \reals^\candidateS}} \quad &\sum_{x, a} \mu(x, a) r(x, a)  \\
\text{u.c.} \quad & \mu \geq 0, \sum_{x,a} \mu(x,a)  = 1 \\
& \forall x, \sum_{a}\mu(x,a) \leq \est_\iepi(x) + \rad(x) \\
& \forall x, \sum_{a}\mu(x,a) \geq \est_\iepi(x) - \rad(x) \\
& \forall x,a, \sum_y \rad(y) \leq \mu(x,a)\rad_l  \\
& \forall i,j, \sum_{x, a} \mu(x, a) \cstr_j^i(x, a) = 0.
\end{aligned}
\end{equation}

The last constraint is the proportional representation constraint. The second to fourth constraints enforce the compatibility of $\mu$ with the $\ell_1$ confidence set. We retrieve the distribution as $\prepi(x) = \sum_{a} \mu(x,a)$, and the policy as:
\begin{align*}
 & \pi_\iepi(x, a) = 
    \begin{cases}
      \frac{\mu(x,a)}{\prepi(x)} & \text{if } \prepi \neq 0\\
      \frac{1}{2} & \text{otherwise\,.}
    \end{cases}  
\end{align*}
Precisely, if some $\tilde{p}_l(x) = 0$, we may set the policy $\piepi(a|x)$ arbitrarily. Since the MDP induced by $\tilde{p}$ is still weakly communicating, and in particular any policy is unichain, the optimal gain in this CMDP is not affected.


We now provide regret and representativeness guarantees.
\begin{theorem}\label{thm:regret}
With probability $\geq 1 - \delta$, the regret of \optalg
satisfies:
\begin{align*}
    \regret(T) = O\big(\sqrt{\ncand T\log(\ncand T/\delta)}\big)\\
    \violate(T) = O\big(\sqrt{\ncand T\log(\ncand T/\delta)}\big).
\end{align*}
Moreover, with probability $1- \delta$, the representation loss of \optalg at horizon $T$ satisfies:
\begin{align*}
    \|\ratio(C_T) - \target\|_{\infty} = O\left(\frac{1}{g^*}\sqrt{\frac{\ncand \log\big(\ncand T / \delta\big)}{T}}\right).
\end{align*}
\end{theorem}

The full proof is in Appendix \ref{sec:proof-regret}. It relies on decomposing regret over episodes, bounding the error on $p$ which decreases over episodes as the confidence sets are refined, and leveraging martingale inequalities on the cumulative rewards. 

Since $\frac{R(T)}{T} = g^* - \frac{N(T)}{T}$, it means that with high probability, the difference between the optimal selection rate and the selection rate of \optalg decreases in $\sqrt{\log(T)/T}$ w.r.t. the horizon $T$. The representation loss decreases at the same speed, meaning that the agent should see enough candidates to accurately estimate $p$, and accept candidates at little cost for representativeness.

Compared to the bound from Proposition \ref{prop:hoeffding-loss}, the cost of not knowing $p$ on representativeness is a $\sqrt{\ncand \log(\ncand)}$ factor. This is due to the estimation of $p$ in the worst case, which is controlled by Lemma \ref{lem:l1-dev}. As we show in our experiments (Sec. \ref{sec:expes}), the impact of $\ncand$ on performance regret (and in turn on sample complexity) is not problematic in our typical citizens' assembly scenario: since there are only a handful of features, our algorithm selects candidates quickly in practice (though representativeness is weakened by not knowing $p$). For specific structures of $p$, we obtain bounds with better scaling in $\ncand$, by controlling each entry of $p$ with Bernstein bounds \cite{maurer2009empirical}, instead 
the $\ell^1$-norm. For completeness, we describe this alternative in Appendix \ref{app:bernstein}.

Interestingly, the representation loss is also inversely proportional to $g^*$, the optimal selection rate in the true CMDP. The reason is that the CMDP constraints do not control the ratios $\ratio_j^i(C_T) = \frac{N_j^i(T)}{N(T)}$, but $N_j^i(T)$ instead (by definition of $R^c(T)$ and $\cstr_j^i$). If $N(T)$ is small, i.e. due to a small selection rate $g$, then $R_j^i(T) = |N_j^i(T) - \target_j^i N(T)|$ is small, but not necessarily $|\frac{N_j^i(T)}{N(T)} - \target_j^i|$: the committee is too small to be representative.





%% file: sections/experiments.tex
\section{Experiments}\label{sec:expes}

The goal of these experiments is to answer the following: 
\textbf{(Q1)} In practice, for which range of committee sizes do our strategies achieve satisfying sample complexity and representation loss?
\textbf{(Q2)} What is the cost of not knowing the distribution $p$ for the sample complexity and representation loss?

\paragraph{Experimental setting} To answer these questions, we use summary data from the 2017 Citizens’ Assembly on Brexit.
The participants were recruited in an offline manner: volunteers could express interest in a survey, and then $53$ citizens were drawn from the pool of volunteers using stratified sampling, in order to construct an assembly that reflects the diversity of the UK electorate. We use summary statistics published in the report \cite{renwick2017considered} to simulate an online recruitment process.

There are $d=6$ features:
the organisers expressed target quotas for 2 ethnicity groups,
2 social classes, 3 age groups, 8 regions, 
2 gender groups and 2 Brexit vote groups (remain, leave). The report also includes the number of people contacted per feature group (e.g., women, or people who voted to remain) and the volunteering rate for each feature group, which we use as probability of volunteering given a feature group. We use Bayes' rule to compute the probabilities of feature groups among 
volunteers, and use 
them as the marginal distributions $\Pr[x^i = j|\text{volunteers}]$ (since we only consider the population of volunteers).  
Since we only have access to the marginals, we compute the joint distribution as if the features were independent, although our model is agnostic to the dependence structure of the joint distribution. In Appendix \ref{app:add_exps}, we present additional experiments with non-independent features, using a real dataset containing demographic attributes. The results are qualitatively similar.

We study \greedyalg with tolerance $\tol=0.02, 0.05$. We run experiments for $K = 50, 100, 150, 250, 500, 1000$, averaged over $50$ simulations. More details are found in App. \ref{app:exp-detail}.

\paragraph{(A1)} We compare \greedyalg and \statioalg, when the distribution $p$ is \emph{known}. Figure~\ref{fig:compare-K} shows that the greedy strategy with $\tol=0.05$ requires $10$ times more samples than \statioalg, and its representation loss is higher as soon as $K\geq 250.$ \greedyalg with lower tolerance $\tol=0.02$ achieves better representation than \statioalg for smaller committees ($K \leq 100$), but the margin quickly decreases with $K$. However, even for small committees, it requires about $100$ times more samples, which is prohibitively expensive. Figure~\ref{fig:compare-K} shows that for \statioalg, the sample complexity grows linearly in the committee size, with a reasonable slope (we need to find $\tau\approx500$ volunteers for a committee of size $K\approx200$). 

\paragraph{(A2)} To corroborate the previously discussed effect of $\ncand$ when $p$ is \emph{unknown}, we evaluate \optalg on different configurations: (1) using only the features ethnicity, social class, and gender ($d=3, \ncand=8$), (2) using all features except regions ($d=5, \ncand=48$). Fig.~\ref{fig:ucrl-K} shows that unlike \statioalg which has full knowledge of $p$, it is for large committee sizes that \optalg reaches low representation loss (below $0.05$ for $K \geq 1500$ in the configuration(1)). This is because $\optalg$ needs to collect more samples to estimate $p$, as discussed in Th. \ref{thm:regret}. For known $p$, the CMDP approach achieves the same representativeness for middle-sized committees (repr. loss $\leq 0.05$ for $K \approx 250$). Hence, comparing the cases of known (Fig. \ref{fig:compare-K}) and unknown distribution $p$ (Fig. \ref{fig:ucrl-K}), the ignorance of $p$ is not costly for sample complexity, but rather for the representation loss which decreases more slowly. 


Consistently with Th. \ref{thm:regret}, we observe that the representation loss is higher when $\candidateS$ is larger ($d=5$). For small and middle-sized committees, the 
loss of \optalg is much worse than \greedyalg's which also works for unknown $p$. For large committees though, the margin is only $0.05$ when $K\gtrsim2000$ and $\tau \approx 3500$ for $\optalg$ (which is $\times 3$ more sample efficient than \greedyalg). In absolute terms, the theoretical regret bounds have a large constant $\sqrt{\card{\candidateS}}$. This constant is likely unavoidable asymptotically because it comes from Lem. \ref{lem:l1-dev}, but our experiments suggest that in the non-asymptotic regime, \optalg performs better than the bound suggests.




\begin{figure}[t]
    \centering
    \includegraphics[width=\linewidth]{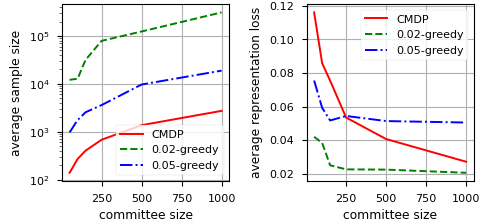}
    \caption{Effect of committee size $K$ on sample complexity and representation loss for different strategies, in the UK Brexit Assembly experiment, using all features. $p$ is \textbf{known}. \label{fig:compare-K}}
\end{figure}

\begin{figure}[t]
    \centering
    \includegraphics[width=\linewidth]{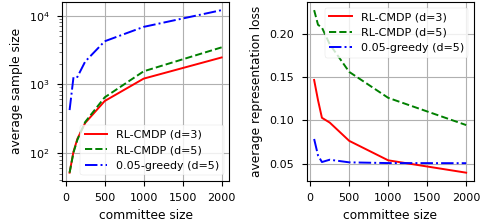}
    \caption{Effect of committee size $K$ on sample complexity and representation loss for \optalg, on data simulated from the UK Brexit Assembly, using $3$ and $5$ features. $p$ is \textbf{unknown}. \label{fig:ucrl-K}}
\end{figure}

%% file: sections/greedy.tex
\section{Details of the algorithms}\label{app:greedy}

For precision, we provide the pseudocode of \greedyalg in Alg. \ref{alg:greedyalg}, and the \statioalg-based strategy in Alg. \ref{alg:statio}.

\begin{algorithm}[t]
\caption{\greedyalg algorithm. \label{alg:greedyalg}}
 \SetKwInOut{Input}{input}\SetKwInOut{Output}{output}
 
 \Input{tolerance $\tol$, committee size $K$, targets $\target$}
 \Output{committee $C_{\tau}$}
 $t\gets 0$, $C_0 \gets \emptyset$\;
 \While{$\card{C_t} < K$}{
 $t \gets t +1$\;
 Observe $x_t \sim p$\;
 \If{$\forall i,j, N_j^i(t) + \indic{x^i_t=j} \leq \ceil{\target_j^i K} + \frac{\tol K}{\size_i - 1}$}{
 $C_t \gets C_{t-1} \cup \{x_t\}$ \tcp*{accept $x_t$}
 $\forall i,j, \, N_j^i(t-1) \gets N_j^i(t) + \indic{x^i_t=j}$
 }
} 
\Return{$C_t$}

\end{algorithm}

\begin{algorithm}[t]
\caption{\statioalg-based strategy. \label{alg:statio}}
 \SetKwInOut{Input}{input}\SetKwInOut{Output}{output}
 \Input{optimal policy $\pi^*$ of \eqref{eq:CMDP-opt1}, committee size $K$}
 \Output{committee $C_\tau$}
 $t\gets 0$, $C_0 \gets \emptyset$\;
 \While{$\card{C_t} < K$ }{
    $t\gets t+1$, observe $x_t \sim p$ and play $a_t \sim \optpi(.|x_t)$ \;
    \lIf{$a_t = 1$}{$C_t \gets C_t \cup \{x_t\}$}
 }
\Return{$C_t$}
\end{algorithm}

We also prove the bound on the representation loss of \greedyalg from Proposition \ref{prop:greedy-loss} in Section \ref{sec:greedy}.

\begin{proof}
For all $i,j$, we have by the if-condition and the termination condition:
\begin{align}
    \ratio^i_{j}(C_\tau) = \frac{N^i_j(\tau)}{K} &\leq \frac{\ceil{\target^i_j K}}{K} + \frac{\tol}{\size_i -1}\nonumber\\
    &\leq \target^i_j + \frac{1}{K} + \frac{\tol}{\size_i -1}\label{eq:greedy1}\\
    &\leq \target^i_j + \frac{\size_i - 1}{K} + \tol.\label{eq:greedy2}
\end{align}

For $i \in \intint{d}$, for $j_0 \in \intint{D_i}$, we have:
\begin{align*}
    &\target_{j_0}^i = 1 - \sum_{j \neq j_0} \target_{j}^i\,, 
    &\ratio_{j_0}^i(C_\tau) = 1 - \sum_{j \neq j_0} \ratio_{j}^i(C_\tau). 
\end{align*}

Combining these observations with \eqref{eq:greedy1}:
\begin{align*}
    \ratio_{j_0}^i(C_\tau)& \geq 1 - \sum_{j \neq j_0}\big(\target^i_j + \frac{1}{K} + \frac{\tol}{\size_i -1}\big)  \\
    & = 1 - \sum_{j \neq j_0}\target^i_j - \frac{\size_i - 1}{K} - \tol \\
    &  =\target_{j_0}^i - \frac{\size_i - 1}{K} - \tol.
\end{align*}

Combining this lower bound with the upper bound \eqref{eq:greedy2}, we have for all $i \in \intint{d}, j_0 \in \intint{D_i}$, $\abs{\ratio_{j_0}^i(C_\tau) - \target_{j_0}^i} \leq \frac{\size_i - 1}{K} + \tol,$ which gives the result.

\end{proof}

%% file: sections/proofs.tex
\section{Proofs}\label{sec:proofs}

\subsection{Proofs of Section \ref{sec:cmdp}}\label{sec:proof-cmdp}

Proof of Lemma \ref{lem:constraint2target}.

\begin{proof}
We have:
\begin{align*}
\sum_{x, a} \pi(x,a)p(x)r_j^i(x,a) &= \expe{\substack{x \sim p \\ a \sim \pi(\cdot|x)}}{r_j^i(x,a)} \\
&= \mathbb{P}^{p, \pi}[a=1, x^i=j],
\end{align*} 
\begin{align*}
    \text{and } \quad g^\pi = \sum_{x,a} \pi(x,a) p(x) r(x,a) = \expe{\substack{x \sim p\\ a \sim \pi(.|x)}}{r(x,a)}\\ = \mathbb{P}^{p,\pi}[a=1].
\end{align*}
The ratio of these two quantities is equal to  $\target_j^i$ by the last constraint of \eqref{eq:CMDP-opt2}. It is also equal to $ \mathbb{P}[x^i=j|a=1],$ which gives the result.

Note that it also holds true for $j = \intint{\size_i}$, since \begin{align*}
&\mathbb{P}[x^i=\size_i|a=1] = 1 - \sum_{j' \in \intint{\size_i - 1}} \mathbb{P}[x^i=j'|a=1] &\text{ and } \\ &\target_{\size_i}^i = 1 - \sum_{j' \in \intint{\size_i - 1}} \target_{j'}^i. 
\end{align*}

\end{proof}

Proof of Proposition \ref{prop:hoeffding-loss}.

\begin{proof} 
For any $t > 0$, we have $$\ratio_j^i(C_t) = \frac{\sum_{s=1}^t \indic{x^i_s=j, a_s=1}}{\sum_{s=1}^t \indic{a_s=1}}.$$

and by Lemma \ref{lem:constraint2target}, we have: $$\expe{}{\indic{x^i=j} \vert a = 1} = \target_j^i.$$






Let $\delta'>0$. Conditionally on any $T \geq K, (a_1, ..., a_T) \in \{0,1\}^T$ s.t. $a_1+ ...+a_T=K$ and $a_T=1$, the draws of $x^i_t|a_t=1$ are independent and thus, by Hoeffding's inequality \cite{hoeffding1994probability}, we have:

\begin{align*}
&\pr{}{|\ratio_j^i(C_T) - \target_j^i| \geq \sqrt{\frac{\log(\frac{2}{\delta'})}{2 N(T)}}\bigg\vert a_1, ..., a_T} \geq 1- \delta'\\
&=\pr{}{|\ratio_j^i(C_T) - \target_j^i| \geq \sqrt{\frac{\log(\frac{2}{\delta'})}{2 K}}\bigg\vert a_1, ..., a_T}.
\end{align*}

Summing up over all such sequences $(a_1,..., a_T)$, we obtain that: 
\begin{align*}
&\pr{}{|\ratio_j^i(C_\tau) - \target_j^i| \geq \sqrt{\frac{\log(\frac{2}{\delta'})}{2 K}}} \geq 1- \delta'.
\end{align*}
The result follows from applying a union bound over all $i \in \intint{d}, j \in \intint{D_i - 1}$ (there are $\tilde{d}$ such $(i,j)$ pairs) and choosing $\delta' = \delta / \tilde{d}$.

\end{proof}

\subsection{Proof of Theorem \ref{thm:regret}}\label{sec:proof-regret}



The following lemma states a standard and useful inequality, which is similar to Lem. 19 in \cite{jaksch2010near}.
\begin{lemma}\label{lem:sum-integrals} Recall that $\nepi$ is the random number of episodes ran by \optalg up until horizon $T$. We have:
\begin{align*}
    \sum_{\iepi=1}^{\nepi} \frac{\card{\durepi}}{\sqrt{\tepi_\iepi - 1}} \leq 2 \sqrt{T}.
\end{align*}
\end{lemma}

\begin{proof}
The proof is similar to that of Lem. 13 in \cite{zanette2019tighter}: we see $\durepi$ as the ``derivative'' of ${\tepi_\iepi}$. Formally, let us define:
\begin{align*}
    &F(x) = \sum_{\iepi=1}^{\floor{x}}\card{\durepi} + \card{E_{\ceil{x}}}(x - \floor{x}) \\
    &f(x) := F'(x) = \card{E_{\ceil{x}}}.
\end{align*}
We first observe that for any integer $\iepi \in \mathbb{N}$, $f(\iepi) = \card{\durepi}$ and $F(\iepi) = \tepi_\iepi.$ Secondly, we have \begin{align*}
     F(x) \leq \sum_{\iepi=1}^{\floor{x}}\card{\durepi} + \card{E_{\ceil{x}}} = \sum_{\iepi=1}^{\ceil{x}}\card{\durepi} = F(\ceil{x}),
\end{align*} and thus:
$$\frac{f(\ceil{x})}{\sqrt{F(\ceil{x})-1}}  \\
    \leq\frac{f(x)}{\sqrt{F(x)-1}}.$$

We derive our bound as follows:
\begin{align*}
    \sum_{\iepi=1}^{\nepi} \frac{\card{\durepi}}{\sqrt{\tepi_\iepi - 1}} &= \sum_{\iepi=1}^{\nepi} \frac{f(\iepi)}{\sqrt{F(\iepi) - 1}} =  \int_{1}^{\nepi}\frac{f(\ceil{x})}{\sqrt{F(\ceil{x})-1}} \,dx \\
    &\leq \int_{1}^{\nepi}\frac{f(x)}{\sqrt{F(x)-1}} \,dx = 2(\sqrt{F(\nepi) - 1}) \\
   & = 2(\sqrt{\tepi_\nepi - 1}) \leq 2 \sqrt{T}.
\end{align*}

\end{proof}

We introduce the following notation: for $f:\candidateS \times \actS \rightarrow \mathbb{R}$, let $f^\pi(x) := \sum_{a} f(x, a) \pi(x,a)$. For all $t > 0$, let $\iepi_t$ denote the episode number at time $t$. The following useful lemma is based on a martingale argument.
\begin{lemma}\label{lem:azuma} Let $f:\candidateS \times \actS \rightarrow \mathbb{R}$. Let $\delta' > 0$. We have:
\begin{align*}
    \pr{}{\sum_{t=1}^T(\langle f^{\piepit}, p\rangle - f(x_t, a_t)) \leq \sqrt{2T \log(1/\delta')}} \geq 1 - \delta'\\
    \pr{}{\sum_{t=1}^T\big\vert\langle f^{\piepit}, p\rangle - f(x_t, a_t)\big\vert \leq \sqrt{2T \log(2/\delta')}} \geq 1 - \delta'.
\end{align*}
\end{lemma}

\begin{proof}
We define the filtration $\filtr_{t} = \sigma(x_1, a_1, ..., x_t, a_t)$ and we first show that the sequence defined by $M_t = \langle f^{\piepit}, p\rangle - f(x_t, a_t)$ is a martingale difference sequence w.r.t. $\filtr_{t}$. $\expe{}{M_t} < \infty$ since the rewards are bounded. Next, the proof that $\expe{}{M_t|\filtr_{t-1}} = 0$ relies on the fact that $\iepi_t$, and in turn the stationary policy $\piepit$, are $\filtr_{t-1}-$measurable.


Therefore, 
\begin{align*}
    \expe{}{\langle f^{\piepit}, p\rangle \big| \filtr_{t-1}} = \langle f^{\piepit}, p\rangle. 
\end{align*}

We also have: 

\begin{align*}
    \expe{}{f(x_t, a_t) \big| \filtr_{t-1}} &= \expe{}{\sum_{x,a} f(x,a) \indic{(x_t, a_t) = (x,a)} \bigg\vert \filtr_{t-1}}\\
    &= \sum_{x,a} f(x,a) \piepit(x,a) = \langle f^{\piepit}, p\rangle.
\end{align*}

Subtracting the two expressions above, we get $\expe{}{M_t|\filtr_{t-1}} = 0$. $(M_t)_{t}$ is thus a Martingale difference sequence, such that $-1\leq M_t \leq 1$. The result follows from Azuma-Hoeffding's inequality.

\end{proof}

We now prove Theorem~\ref{thm:regret}.

\begin{proof} 
We define $\mathcal{E} = \mathcal{E}_1 \cap\mathcal{E}_2  \cap\mathcal{E}_3 $ to be the ``good event'', with:
\begin{align*}
    &\mathcal{E}_1 = \{\forall \iepi \geq 1, \prepi \in B_\iepi\},\\
    &\mathcal{E}_2 = \{\sum_{t=1}^T(\langle r^{\piepit}, p\rangle - r(x_t, a_t)) \leq \sqrt{2T \log(3/\delta)}\}, \\
    &\mathcal{E}_3 = \left\{\forall i,j, \quad \sum_{t=1}^T\big\vert\langle {\cstr_j^i}^{\piepit}, p\rangle - {\cstr_j^i}(x_t, a_t)\big\vert \leq \sqrt{2T \log\big(\frac{6\tilde{d}}{\delta}\big)}\right\}.
\end{align*}

By Lemma \ref{lem:l1-dev}, we have
\begin{align}\label{eq:probinterval}
    \pr{}{\exists \iepi \geq 1, \prepi \in B_\iepi} \geq 1- \frac{\delta}{3}. 
\end{align}
Combining \eqref{eq:probinterval} with Lemma \ref{lem:azuma} and using union bounds, $\pr{}{\mathcal{E}} \geq 1 - \delta$. From now on, we assume that the good event $\mathcal{E}$ holds true.

\paragraph{Performance regret} We start by upper bounding the performance regret $\regret(T)$. Let $\Delta_\iepi = \sumepi (g^* - r(x_t, a_t))$ be the regret of episode $\iepi$. Let $(\pi_\iepi, \tilde{p}_\iepi)$ be the solution of the optimistic CMDP \eqref{eq:OptCMDP-1} at episode $\iepi$. Since $(\pi^*, p)$ is feasible for \eqref{eq:OptCMDP-1}, then $g^* \leq g^{\prepi,\piepi}$. We also note that $$g^{\prepi,\piepi} = \sum_{x,a} r(x,a)\prepi(x)\piepi(x,a) = \sum_{x} r^{\piepi}(x)\prepi(x).$$

Therefore, we have:
\begin{equation}\label{eq:bound-epi-reg}
\begin{aligned}
    \Delta_\iepi &\leq \sumepi (g^{\prepi,\piepi} - r(x_t, a_t))\\
    &= \sumepi (\sum_{x} r^{\piepi}(x)\prepi(x) - r(x_t, a_t))\\
    &= \sumepi \sum_{x} r^{\piepi}(x)(\prepi(x) - p(x))\\
    & \quad + \sumepi(\sum_{x} r^{\piepi}(x)p(x) - r(x_t, a_t))\,
\end{aligned}
\end{equation}

Using Hölder's inequality and the fact that $\|r\|_{\infty} = 1$, the first term can be bounded by $\card{\durepi} \|\prepi - p\|_{1}$. By validity of the confidence intervals under event $\mathcal{E}$:
\begin{equation}
\begin{aligned}
    \|\prepi - p\|_{1} \leq 2 \radepi \leq \frac{2 \sqrt{2\ncand\log\big(6\ncand T (T-1) / \delta\big)}}{\sqrt{\tepi_\iepi -1}}
\end{aligned}
\end{equation}

Summing up over episodes $\iepi=1,..., \nepi$:
\begin{equation}
\begin{aligned}
    \regret(T) &\leq 2 \sqrt{2\ncand\log\big(\frac{6\ncand T (T-1)}{ \delta}\big)}\sum_{\iepi=1}^{\nepi}\frac{\card{\durepi} }{\sqrt{\tepi_\iepi -1}} \\
    &+ \sum_{t=1}^T(\sum_{x} r^{\piepit}(x)p(x) - r(x_t, a_t)).
\end{aligned} 
\end{equation}

We bound the first sum using Lemma \ref{lem:sum-integrals}. The second term can be bounded as in Lemma \ref{lem:azuma} because $\mathcal{E}_2$ holds true. This gives us the resulting bound which holds under $\mathcal{E}$:
\begin{equation*}
\begin{aligned}
    \regret(T) &\leq 4 \sqrt{\ncand\log\big(\frac{6\ncand T (T-1)}{ \delta}\big)T} + \sqrt{2T \log\big(\frac{3}{\delta}\big)}.
\end{aligned} 
\end{equation*}

\paragraph{Cost of constraint violations} The proof for the cost of constraint violations is very similar. Let us bound $\regret_j^i(T) := \sum_{t=1}^T |\cstr_j^i(x_t, a_t)|$ for all $i,j$. We briefly drop the sub/superscripts $i,j$.

At each episode $\iepi$, since $(\piepi, \prepi)$ is a solution of \eqref{eq:OptCMDP-1}, we have $h^{\prepi, \piepi}=0$, and thus $\sum_{x,a}\cstr(x,a)\piepi(x,a)\prepi(x) = \sum_{x}\cstr^{\piepi}(x)\prepi(x) = 0$. Therefore, we have:

\begin{align*}
    \abs{\sum_{t=1}^T \cstr(x_t, a_t)} &= \abs{\sum_{\iepi=1}^{\nepi}\big(\sumepi \cstr(x_t, a_t) - \sum_{x} \cstr^{\piepi} (x)\prepi(x)\big)} \\ &\leq\bigg\vert \sum_{\iepi=1}^{\nepi}\sumepi\sum_{x} \cstr^{\piepi} (x)(p(x) - \prepi(x))\\
    & \quad \quad + \sum_{\iepi=1}^{\nepi}\big(\sumepi\cstr(x_t, a_t) - \sum_{x} \cstr^{\piepi}(x)p(x) \big)\bigg\vert\\
    &\leq \sum_{\iepi=1}^{\nepi}\sumepi\abs{\sum_{x} \cstr^{\piepi} (x)(p(x) - \prepi(x))}\\
    & \quad \quad + \abs{\sum_{\iepi=1}^{\nepi}\big(\sumepi\cstr(x_t, a_t) - \sum_{x} \cstr^{\piepi}(x)p(x) \big)}\\
    &\leq \sum_{\iepi=1}^{\nepi}\card{\durepi}\|\cstr^{\piepi}\|_{\infty}\|p - \prepi\|_1\\
    & \quad \quad + \abs{\sum_{t=1}^T \bigg(\cstr(x_t, a_t) - \sum_{x} \cstr^{\piepit}(x)p(x)\bigg)},\\
\end{align*}
where the first part of the last inequality is again by Hölder's inequality. Similarly to the performance regret,  the first term is bounded using the validity of confidence intervals under the good event $\mathcal{E}$ and Lemma \ref{lem:sum-integrals}, and the second term is bounded by the martingale argument using Lemma \ref{lem:azuma}. Hence, under $\mathcal{E}$ we have for any $i,j$:
\begin{equation*}
\begin{aligned}
    R_j^i(T) &\leq 4 \sqrt{\ncand\log\big(\frac{6\ncand T (T-1)}{ \delta}\big)T} + \sqrt{2T \log\big(\frac{6\tilde{d}}{\delta}\big)}.
\end{aligned} 
\end{equation*}

And thus the same bounds holds for $\violate(T) = \max_{i,j} R_j^i(T)$.

\paragraph{Representation loss}
We may now derive the bound on representation loss.

Let $f(T) = O\big(\sqrt{\ncand \log(\ncand T/\delta)}\big)$. The regret bounds imply that with $1- \delta$:
\begin{align*}
    R(T) =& g^*T - N(T) \leq f(T) \Rightarrow N(T) \geq g^*T - f(T)\\
    \frac{R^c(T)}{N(T)} =& \max_{i,j} \abs{\frac{N_j^i(T)}{N(T)} - \target_j^i \frac{N(T)}{N(T)}}\leq \frac{f(T)}{N(T)}\\
    \text{i.e., } & \|\ratio(C_T) - \target\|_{\infty} \leq \frac{f(T)}{N(T)}.
\end{align*}
Therefore, using $N(T) \geq 1$, we have: 
\begin{align*}
    \|\ratio(C_T) - \target\|_{\infty}&\leq \frac{f(T)}{\max(1, g^*T - f(T))} \\
    &= O\bigg(\sqrt{\frac{\ncand \log(\ncand T/\delta)}{{g^*}^2T}}\bigg).
\end{align*}
\end{proof}

%% file: sections/bernstein.tex
\section{Alternative to \optalg with Bernstein bounds}\label{app:bernstein}

We present \optalgb, an alternative to \optalg which uses Bernstein empirical bounds \cite{maurer2009empirical}.

At each episode $l$, the algorithm estimates the distributions by $\est_\iepi(x) = \frac{n_{\tepi_\iepi-1}(x)}{\tepi_\iepi-1}$ and maintains confidence intervals $[\lcb_\iepi(x), \ucb_\iepi(x)]$. These are built using Bernstein's empirical inequality \cite{maurer2009empirical}, which implies that there exists constants $B_1, B_2$ such that with probability $\geq 1 - \frac{\delta}{3}$, for each $l\geq 1$ and $x \in \candidateS$, 
\begin{align}\label{eq:bernstein}
    |p(x) - \est_\iepi(x)| \leq B_1 \sqrt{\frac{\hat{\sigma}_\iepi^2(x) \log(\frac{6\ncand\tepi_\iepi}{\delta} )}{1 \land(\tepi_\iepi-1)}} + B_2 \frac{\log(\frac{6\ncand\tepi_\iepi}{\delta})}{1 \land (\tepi_\iepi-1)},
\end{align} 

where $\hat{\sigma}_\iepi(x) = \sqrt{\est_\iepi(x) (1- \est_\iepi(x))}$.

Following e.g. \cite{efroni2020exploration}, we re-write \eqref{eq:OptCMDP-1} as an extended LP by introducing the state-action occupation measure $  \mu(x,a) = \pi(x,a) p(x)$.

\begin{equation}\label{eq:OptCMDP-2-alternative}
\begin{aligned}
\max_{\mu \in \reals^{\candidateS \times \actS }} \quad &\sum_{x, a} \mu(x, a) r(x, a)  \\
\text{u.c.} \quad & \mu \geq 0, \sum_{x,a} \mu(x,a)  = 1 \\
& \forall x, \sum_{a}\mu(x,a) \leq \ucb_\iepi(x) \\
& \forall x, \sum_{a}\mu(x,a) \geq \lcb_\iepi(x) \\
& \forall i,j, \sum_{x, a} \mu(x, a) \cstr_j^i(x, a) = 0.
\end{aligned}
\end{equation}

The second to fourth constraints enforce the compatibility of $\mu$ with the confidence intervals. Controlling each entry of $p$ with Bernstein bounds instead of the $\ell^1$-norm allows for a simpler optimization problem than the extended LP \eqref{eq:OptCMDP-2}. We get the following regret bound:

\begin{theorem}[Regret guarantees]\label{thm:regret-bernstein}
With probability $\geq 1 - \delta$, the regret of \optalgb
satisfies:
\begin{align*}
    \regret(T) = O\big(\sqrt{\ncand T\log(\ncand T/\delta)} + \ncand \log(\ncand T/\delta)^2\big)\\
    \violate(T) = O\big(\sqrt{\ncand T\log(\ncand T/\delta)} + \ncand \log(\ncand T/\delta)^2\big).
\end{align*}

With probability $\geq 1 - \delta$, the representation loss satisfies:
\begin{align*}
    &\|\ratio(C_T) - \target\|_{\infty} \\
    &= O\left(\frac{1}{g^*}\sqrt{\frac{\ncand \log\big(\ncand T / \delta\big)}{T}} + \frac{\ncand \log(\ncand T/\delta)^2}{g^*T} \right).
\end{align*}

\end{theorem}

When using Bernstein bounds, the representation loss carries $O(\ncand \log(\ncand T/\delta)^2)$. This factor but has a bigger scaling with $\ncand$, but decreases rapidly in $\frac{\log(T)^2}{T}$.

The Bernstein version of $\optalg$ may be advantageous for some candidate distributions $p$. For example, if the support $\supp$ of $p$ is very small compared to $\candidateS$, the first term in the Bernstein empirical inequality \eqref{eq:bernstein} is equal to zero for all $x$ outside the support. Therefore, the representation loss scales as:
\begin{align*}
    &\|\ratio(C_T) - \target\|_{\infty} \\
    &= O\left(\frac{1}{g^*}\sqrt{\frac{\card{\supp} \log\big(\card{\supp} T / \delta\big)}{T}} + \frac{\ncand \log(\ncand T/\delta)^2}{g^*T} \right),
\end{align*}
where $\card{\supp} \ll \ncand$. Thus, the second term with fast decrease in $\frac{\log(T)^2}{T}$ controls the bound on representation loss.

\subsection{Proofs}

The following lemma states a useful inequality akin to Lemma \ref{lem:sum-integrals}.
\begin{lemma}\label{lem:sum-integrals-2} We have:
\begin{align*}
    \sum_{\iepi=1}^{\nepi} \frac{\card{\durepi}}{\tepi_\iepi -1} \leq \log(T) 
\end{align*}
\end{lemma}

\begin{proof} The proof is similar to Lem. 13 in \cite{zanette2019tighter}. Using the same notation as in the proof of Lemma \ref{lem:sum-integrals},
\begin{align*}
    \sum_{\iepi=1}^{\nepi} \frac{\card{\durepi}}{\tepi_\iepi - 1} &= \sum_{\iepi=1}^{\nepi} \frac{f(\iepi)}{F(\iepi) - 1} =  \int_{1}^{\nepi}\frac{f(\ceil{x})}{F(\ceil{x})-1} \,dx \\
    &\leq \int_{1}^{\nepi}\frac{f(x)}{F(x)-1} \,dx = \log(F(\nepi) - 1) \\
   & = \log(\tepi_\nepi - 1) \leq \log{T}.
\end{align*}
\end{proof}
We now prove Theorem \ref{thm:regret-bernstein}.

\begin{proof}
We re-use the same steps and notation as for the proof of Theorem \ref{thm:regret}. 

Here instead, $\mathcal{E}_1$ is the event such that the confidence intervals are valid \eqref{eq:bernstein}. Under the high-probability good event $\mathcal{E} = \mathcal{E}_1 \cap \mathcal{E}_2 \cap \mathcal{E}_3$, we thus have:

$|\piepi(x) - p(x)| \lesssim \sqrt{\frac{\hat{p}_l(x)(1-\hat{p}_l(x))b_{\delta, T}}{\tau_l-1}} + \frac{b_{\delta, T}}{\tau_l-1}$

where $b_{\delta, T} = \log(\frac{6\ncand T}{\delta})$.

In the following, the first inequality is by validity of the Bernstein confidence intervals under $\mathcal{E}$, and the second inequality is by Cauchy-Schwarz's inequality:
\begin{equation}\label{eq:bern-proof-1}
\begin{aligned}
&\sumepi \sum_{x} r^{\piepi}(x)(\prepi(x) - p(x)) \\
    &\leq \sumepi \sum_x r^{\piepi}(x) \sqrt{\frac{\hat{p}_l(x)(1-\hat{p}_l(x))b_{\delta, T}}{\tau_l-1}} \\
    & \quad + \sumepi \frac{b_{\delta, T}}{\tau_l-1}  \underbrace{\sum_x r^{\piepi}(x)}_{\leq \ncand}\\
    &\leq \sumepi \sqrt{ \underbrace{(\sum_x 1 - \hat{p}_l(x))}_{\leq \ncand} \underbrace{(\sum_x \hat{p}_l(x) r^{\piepi}(x) b_{\delta, T})}_{\leq b_{\delta, T}}} \sqrt{\frac{1}{\tau_l - 1}}\\
     &\quad + \sumepi \frac{\ncand b_{\delta, T}}{\tau_l - 1}
\end{aligned}
\end{equation}

By Lemmas \ref{lem:sum-integrals} and \ref{lem:sum-integrals-2}, we have:
\begin{align*}
    &\sqrt{\ncand b_{\delta, T}} \sum_{\iepi=1}^{\nepi} \frac{|E_l|}{\sqrt{\tau_l - 1}} \leq 2\sqrt{\ncand b_{\delta, T} T}\\
    &\ncand b_{\delta, T} \sum_{\iepi=1}^{\nepi} \frac{|E_l|}{\tau_l - 1} \leq \ncand b_{\delta, T}\log(T).
\end{align*}

Summing up over episodes in inequality \eqref{eq:bern-proof-1} and plugging in the above inequalities gives the desired bound by following the steps of the proof of Theorem \ref{thm:regret}.
\end{proof}





%% file: sections/additional_exps.tex
\section{Experiments}
\subsection{Details on the Brexit experiments} \label{app:exp-detail}

We provide in Table \ref{tab:brexit-stats} the target vectors $(\target_j^i)_{i,j}$ and marginal distributions $(\mathbb{P}^p[x^i=j])_{i,j}$ extracted from the Citizens' Assembly on Brexit report \cite{renwick2017considered}.\footnote{ \url{https://citizensassembly.co.uk/wp-content/uploads/2017/12/Citizens-Assembly-on-Brexit-Report.pdf}, pages 28-32.} The report includes the volunteering rates for each feature group, i.e. $\Pr[\text{volunteer}| x^i=j]$. To compute the marginal distributions $(\Pr[x^i=j | \text{volunteer}])_{i,j}$, we thus use Bayes' rule to compute the probability of each feature group among the volunteer population\footnotemark, that is: \begin{align*}
  \mathbb{P}^p[x^i=j] &= \Pr[x^i=j | \text{volunteer}]\\ &=  \frac{\Pr[\text{volunteer}| x^i=j]\Pr[x^i=j]}{\Pr[\text{volunteer}]}.   
\end{align*}
\footnotetext{In doing so, we notice that the probability of finding non-voter volunteers is almost zero, hence we only consider ``remain'' and ``leave'' for the feature Brexit vote. Indeed, the report states ``The only target that proved impossible to meet was that for non-voters in the 2016 referendum.'' p.28.}

We often have $\target_j^i \neq \mathbb{P}^p[x^i=j]$. For example, compared to the age target, we are less likely to find younger people ($\leq34$ years old) among volunteers. For gender, while the target was gender parity, we are much less likely to find women than men in the volunteer population.

\begin{table}[]
\centering
\begin{tabular}{|l|l|l|}
\hline
 & Targets & Marginals \\
 \hline
Ethnicity & 0.860 / 0.140 & 0.863 / 0.136 \\
\hline
Social class & 0.550 / 0.450 & 0.556 / 0.444 \\
\hline
Age & 0.288 / 0.344 / 0.367 & 0.154 / 0.432 / 0.414 \\
\hline
Region & \begin{tabular}[c]{@{}l@{}}0.233 / 0.160 / 0.093\\  / 0.134 / 0.222 / 0.047\\  / 0.082 / 0.028\end{tabular} & \begin{tabular}[c]{@{}l@{}}0.179 / 0.155 / 0.090 \\ / 0.117 / 0.211 / 0.073 \\ / 0.154 / 0.021\end{tabular} \\
\hline
Gender & 0.507 / 0.493 & 0.384 / 0.616 \\
\hline
Brexit vote & 0.481 / 0.519 & 0.565 / 0.434\\
\hline
\end{tabular}
\caption{Target quotas (from the report) and marginal distribution (computed using Bayes' rule) for the Brexit experiment.}
\label{tab:brexit-stats}
\end{table}

For our experiments presented in Section \ref{sec:expes}, we used Python and the CPLEX LP solver, and a machine with Intel Xeon Gold 6230 CPUs, 2.10 GHz, 1.3 MiB of cache. 

\subsection{Experiments with dependent features}\label{app:add_exps}

The goal of these experiments is to answer the following: what is the impact of the dependence structure of the joint feature distribution $p$ on the sample complexity and representation loss of our algorithms? Since we may only retrieve marginal distributions from the Citizen's Assembly on Brexit report, we keep the target quotas on each feature but simulate joint feature distributions from another dataset with demographic attributes, the standard Adult Census Income dataset \cite{Dua:2019}. 

The Adult dataset consists of approximately $49.000$ entries of subjects in the US, each with 14 demographic features and a binary label indicating whether a subject’s income is above or below $50$K USD. We only keep features that can be mapped to our Brexit Citizen's Assembly example: gender, age, ethnicity and income, which we use in lieu of social class. We do not consider proportional representation for region and Brexit vote since there are no such features in the Adult dataset. In our preprocessing of the Adult dataset, we create the same three age categories ($<$35, 35-54, $>$54), the same two ethnicity groups (white / non-white) and we use the binary income variable as a proxy for social class, by assigning $>50$K to upper class and $\leq50$K to lower class. This leaves us with $4$ features with $2, 2, 2, 3$ possible values.

To create dependencies between features, we consider two graphical structures shown in Figure \ref{fig:BN}, and for each we fit a Bayesian network to the dataset to generate a model of the joint distribution $p(x)$. We consider one structure with little dependence, and one structure with strong dependence between features.

\begin{figure}
\centering
\begin{subfigure}{.45\linewidth}
    \centering
    \includegraphics[width=\linewidth]{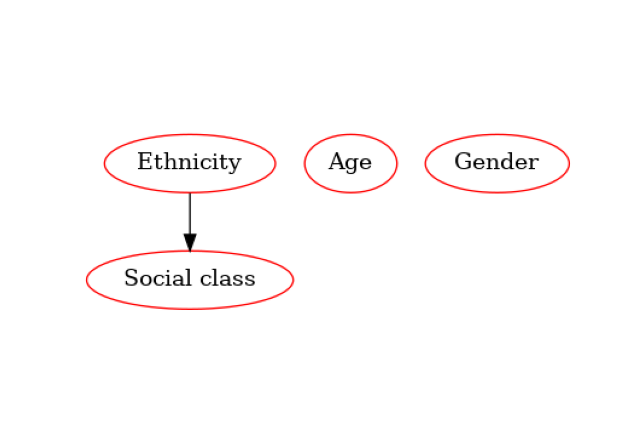}
    \caption{Structure 1: weak dependence.  \label{fig:BNstruct01}}
\end{subfigure}
\begin{subfigure}{.45\linewidth}
    \centering
    \includegraphics[width=\linewidth]{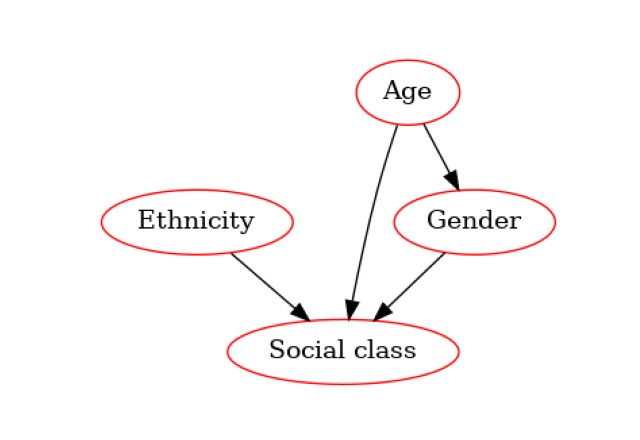}
    \caption{Structure 2: strong dependence.  \label{fig:BNstruct02}}
\end{subfigure}
\caption{Bayesian network structures for the Census Income dataset. \label{fig:BN}}
\end{figure}
\begin{figure}[t]
    \centering
    \includegraphics[width=\linewidth]{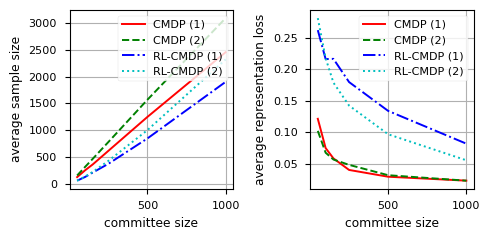}
    \caption{Effect of committee size $K$ on sample complexity and representation loss for \statioalg (\textbf{known} $p$) and \optalg (\textbf{unknown} $p$), on the two different Bayesian networks (1) and (2) fitted on the Census Income dataset. \label{fig:adult-K-plots}}
\end{figure}

Figure \ref{fig:adult-K-plots} shows that both when $p$ is known, the sample complexity is higher when there is more dependence (Bayesian network (2)) between features, but the representation loss is the same. When $p$ is unknown, the representation loss is lower for structure (2) with more dependence, than structure (1) with little dependence, but the sample size is higher for (2). For structure (2), the representation loss is low ($\approx 0.07$) for $K=1000$. Importantly, it implies that in practice, the representation loss is much lower than the worst case bound given by Theorem \ref{thm:regret}.

%% file: sections/example-cmdp.tex
\section{Detailed example for Section 4}\label{appendix:cmdp}

We take the same attributes and same distribution as in Table \ref{tab:example-dist}, with $\epsilon' = \nicefrac16$:

\begin{table}[]
\centering
\begin{tabular}{|l|ll|}
\hline
gender \textbackslash{} age & S & J \\
\hline
M & $\nicefrac13$ & $\nicefrac14$ \\
F & $\nicefrac14$ & $\nicefrac16$\\
\hline
\end{tabular}
\label{tab:example-cmdp}
\end{table}

The target vectors are $\rho^{\text{gender}} = (\nicefrac12,\nicefrac12)$ and $\rho^{\text{age}} = (\nicefrac12,\nicefrac12)$, that is, an ideal committee contains as many women as men and as many senior than junior. 

We solve the linear program

\begin{equation}\label{eq:ex-CMDP}
\begin{aligned}
\max \quad & \frac{\pi(MS,1)}{3} + \frac{\pi(FS, 1)}{4} + \frac{\mu(MJ,1)}{4} + \frac{\mu(FJ, 1)}{6} \\ \\
\text{u.c.} \quad & \frac{\pi(MS,1)}{3} + \frac{\pi(FS, 1)}{4} = \frac{\mu(MJ,1)}{4} + \frac{\mu(FJ, 1)}{6} \\
& \frac{\pi(MS,1)}{3} + \frac{\mu(MJ,1)}{4} = \frac{\pi(FS, 1)}{4} + \frac{\mu(FJ, 1)}{6} \\
\end{aligned}
\end{equation}

Its solution is

\begin{equation}\label{eq:ex-CMDP4}
\begin{aligned}
\pi^*(MS,1) = \nicefrac12\\
\pi^*(FJ,1) = 1\\
\pi^*(MJ,1) = 1\\
\pi^*(FS,1) = 1
\end{aligned}
\end{equation}

Thus, each time the current volunteer is a senior male, we select him with probability \nicefrac12; all other volunteers are selected with probability 1. The expected final composition of the pool is 30\% of junior male, 30\% of senior female, 20\% of junior female and 20\% of senior male. As the policy selects in average \nicefrac56 of the volunteers, the expected time until we select $K$ candidates is $\mathbb{E}^{p,\pi^*}[\tau] = (\nicefrac65) K$.